\documentclass[conference]{IEEEtran}
\IEEEoverridecommandlockouts
\usepackage{cite}
\usepackage{amsmath,amssymb,amsfonts}
\usepackage{algorithmic}
\usepackage{graphicx}
\usepackage{textcomp}
\usepackage{xcolor}
\usepackage{times}
\usepackage{mathtools}
\usepackage{bm}
\usepackage{amsthm}
\def\BibTeX{{\rm B\kern-.05em{\sc i\kern-.025em b}\kern-.08em
    T\kern-.1667em\lower.7ex\hbox{E}\kern-.125emX}}

\usepackage{flushend}

\newtheorem{theorem}{Theorem}
\newtheorem{lemma}{Lemma}
\newtheorem{proposition}{Proposition}
\newtheorem{corollary}{Corollary}

\usepackage{multibib}
\newcites{apndx}{References in Appendix}

\let\emptyset\varnothing

\newcommand{\prob}[1]{\mathbb{P}\left\{ #1 \right\} }

\newcommand{\E}[1]{\mathbb{E}\left\{ #1 \right\} }
\newcommand{\loss}[1]{\textsf{Loss}\left (#1\right)}
\newcommand{\re}[1]{\textcolor{black}{#1}}

\DeclareMathOperator*{\argmax}{arg\,max}

\begin{document}

\title{Learning, compression, and leakage:\\
Minimising classification error via\\
meta-universal compression principles
\thanks{F.R. is also with the Psychedelic Research Center, and with the Centre for Complexity Science, Imperial College London. F.R. was supported by the Ad Astra Chandaria foundation, and by the European Union’s H2020 research and innovation programme under the Marie Sk\l{}odowska-Curie grant agreement No. 702981. P.M. was funded by the Wellcome Trust (grant no. 210920/Z/18/Z).}
}

\author{\IEEEauthorblockN{Fernando E. Rosas}
\IEEEauthorblockA{\textit{Data Science Institute} \\
\textit{Imperial College London}\\
London, UK \\
\texttt{f.rosas@imperial.ac.uk}}
\and
\IEEEauthorblockN{Pedro A.M. Mediano}
\IEEEauthorblockA{\textit{Department of Psychology} \\
\textit{University of Cambridge}\\
Cambridge, UK \\
\texttt{pam83@cam.ac.uk}}
\and
\IEEEauthorblockN{Michael Gastpar}
\IEEEauthorblockA{\textit{School of Computer and Communication Sciences} \\
\textit{\'{E}cole Polytechnique F\'{e}d\'{e}rale de Lausanne}\\
Lausanne, Switzerland \\
\texttt{michael.gastpar@epfl.ch}}
}

\maketitle

\begin{abstract}

Learning and compression are driven by the common aim of identifying and exploiting statistical regularities in data, which opens the door for fertile collaboration between these areas. A promising group of compression techniques for learning scenarios is \textit{normalised maximum likelihood} (NML) coding, which provides strong guarantees for compression of small datasets --- in contrast with more popular estimators whose guarantees hold only in the asymptotic limit. Here we \re{consider a} %
NML-based decision strategy for supervised classification problems, and show that it attains heuristic PAC learning when applied to a wide variety of models. Furthermore, we show that the misclassification rate of our method is upper bounded by the \textit{maximal leakage}, a recently proposed metric to quantify the potential of data leakage in privacy-sensitive scenarios. 

\end{abstract}


\section{Introduction}

Since compression and learning are both based on exploiting statistical regularities of the data, it is often possible to leverage compression techniques to enable novel learning methods. Examples of successful translations abound in the literature, 
including the use of universal compression methods such as Context Tree Weighting~\cite{willems1995context} for predicting time series via variable-order Markov chains~\cite{begleiter2004prediction}.

Among the literature on universal compression, the %
work of Jorma Rissanen and the Minimum Description Length (MDL) community is particularly well-suited for statistical learning. 
There are two particularly attractive aspects of the MDL philosophy from a learning perspective (c.f. ~\cite{rissanen1989stochastic,grunwald2007minimum}): a focus on the data itself and not on assumptions about related probabilistic models, and an emphasis on estimators that have useful properties for finite sample sizes. These ideas lead to the use of \emph{normalised maximum likelihood} (NML) codes, previously introduced by Shtar'kov~\cite{shtar1987universal}, to develop universal compression methods~\cite{Rissanen1996fisher}. NML distributions provide minimax optimal  %
compression features for finite sample sizes --- in contrast to e.g. distributions obtained via maximum likelihood estimation that only have guarantees in the asymptotic regime.

Despite of their attractive properties, NML distributions have not been much %
explored in the statistical learning literature. An important exception is the work reported in Refs.~\cite{Fogel2019individualdata,bibas2019regression,bibas2019deep}, which leverages conditional NML (cNML) distributions --- originally introduced by Roos \& Rissanen~\cite{roos2008sequentially} --- to address a supervised learning setting. The favourable properties of cNML-based learning strategies have been demonstrated for the cases of linear regression~\cite{bibas2019regression} and deep neural networks~\cite{bibas2019deep}. Unfortunately, the available theoretical guarantees for the learnability of cNML models are still limited. %

Another important contribution of Rissanen was the development of the notion of \emph{stochastic complexity}, a metric of model complexity that refines well-known model selection procedures such as the Akaike and Bayesian Information Criteria~\cite{Rissanen1996fisher}. 
Stochastic complexity has a remarkable similarity to \emph{maximal leakage}, a measure introduced in Ref.~\cite{leakage} to quantify leakage risk in privacy-sensitive scenarios. %
This formal similarity is particularly intriguing given the connection that exist between data privacy and learning: as privacy-preserving algorithms only process general properties of datasets without focusing on particular data samples~(see \cite{Rassouli2019}), they are less likely to fall prey to overfitting. This idea was first developed in the context of differential privacy~\cite{dwork2014algorithmic}, and recent reports have shown that maximal leakage can be used to bound the generalisation error in supervised learning scenarios~\cite{Issa2019strengthened,esposito2019new}.

The goal of this paper is to establish a rigorous link between supervised learning, NML methods, and maximal leakage. For this, we \re{employ a} NML-based decision strategy based on meta-universal compression principles~\cite[Ch. 11.2]{grunwald2007minimum}, where the model is dynamically adapted according to the training data. 
We provide an upper bound, based on maximal leakage, to the performance gap between our NML strategy and the (optimal) MAP criterion (Theorem~\ref{prop:bound_nice}). Furthermore, using this bound we show
that our NML strategy possesses strong learning guarantees that hold in various contexts (Theorem~\ref{teo:pac} and Proposition~\ref{prop:non}). 
Importantly, while most of the MDL literature is based on logarithmic losses (including~\cite{Fogel2019individualdata,bibas2019regression,bibas2019deep}), our approach quantifies performance in terms of classification accuracy, 
\re{which is} a more natural metric for supervised learning scenarios.

The rest of the paper is structured as follows. Section~\ref{sec:preliminars} introduces our supervised learning scenario and discusses fundamental notions of universal compression and information leakage. Section~\ref{sec:main_results} presents our main technical results, and Section~\ref{sec:conc} summarises our conclusions. The Appendices provide the proofs of our results, illustrate the findings in a simple scenario, and discuss some implementation issues.

\section{Preliminaries}\label{sec:preliminars}

\subsection{Scenario}
Let us consider a classification task where one needs to decide which class %
$Y\in\mathcal{Y}=\{c_1,\dots,c_K\}$ %
a given observation $X\in\mathcal{X}$ belongs to. 
A \emph{hypothesis} is a (possibly stochastic) mapping $h:\mathcal{X}\to\mathcal{Y}$, whose performance is measured using the \emph{0--1 loss function} given by
\begin{equation}
\loss{y,\tilde{y}} = 
\begin{cases}
1 \qquad \text{if} \quad y \neq \tilde{y}, \\
0 \qquad \text{otherwise.}
\end{cases}\nonumber
\end{equation}
The misclassification probability of $x$ under $h$ is calculated as
\begin{align}
\texttt{E}(h; x )  \coloneqq& \mathbb{E}\{ \loss{ Y, h(X) } | X=x\} \nonumber\\
=& \prob{ Y \neq h(X) | X=x} \nonumber\\
=& 1 - f( h(x) |x),
\end{align}
where $f(y|x)$ is the conditional probability of $\{Y=y\}$ given $\{X=x\}$. The misclassification rate of $h$ is defined as $\texttt{E}(h)  \coloneqq \mathbb{E}\{ \texttt{E}( h; X) \} = \mathbb{E}\{ \loss{ Y,h(X) } \}$. 
The well-known \textit{maximum-a-posteriori} (MAP) rule, defined as\footnote{In case there is more than one value of $y$ that maximises \eqref{eq:MAP}, $h_\text{MAP}(x)$ assigns one of them randomly.}
\begin{equation}\label{eq:MAP}
h_\text{MAP}(x) \coloneqq \argmax_{y\in \mathcal{Y}} f(y|x),
\end{equation}
can be shown to attain a minimal misclassification rate given by
$\texttt{E}(h_\text{MAP}; x) 
= 1 - \max_{y\in\mathcal{Y}} f(y|x)$~\cite{poor2013introduction}. %
Unfortunately, to build $h_\text{MAP}$ one needs precise knowledge of $f(y|x)$,  which is rarely available in most scenarios of practical interest.

Consider now $n$ available samples for training denoted by $z_1=(x_1,y_1),\dots,z_n=(x_n,y_n)$, and denote the whole dataset by $z^n=(z_1,\dots,z_n)$. Hypotheses that are built on training data correspond to functions $h:\mathcal{X}\times \mathcal{Z}^n \to \mathcal{Y}$, where $\mathcal{Z}:= \mathcal{X}\times \mathcal{Y}$. Then, a hypothesis $h(x,z^n)$ can be equivalently expressed as %
\begin{equation}\label{eq:q}
h_q(x,z^n) = \argmax_{y\in\mathcal{Y}} q(y|x,z^n),
\end{equation}
where $q(y|x,z^n)$ is a (possibly not unique) suitable conditional probability distribution. 
The misclassification rate of $h_q$ is %
\begin{align}
\texttt{E}(h_q;x,z^n) &\coloneqq \prob{Y\neq h_q(X,Z^n) | X=x,Z^n=z^n} \\
&= 1 - f( h_q(x,z^n) |x).
\end{align}
Due to the optimality of $h_\text{MAP}$, $\texttt{E}(h_{\text{MAP}}; x) \leq \texttt{E}(h_q;x,z^n)$ holds for any hypotheses given by $q(y|x,z^n)$. 

\subsection{Universal compression}
\label{sec:universal_compression}

While elementary compression algorithms consider data coming from a single information source (i.e.  i.i.d. data generated from symbols in the alphabet $\mathcal{Y}$ according to a given probability distribution $p(y)$), universal compression approaches aim to 
be suitable to compress data \re{with respect to }%
a statistical model class $\mathcal{M}$ --- understood as a collection of probability distributions. %
The goal is to %
build distributions $q$ that attain %
low values of
\begin{equation}
\text{REG}_\text{max}(\mathcal{M}, q) \coloneqq \sup_{p\in \mathcal{M}} \max_{y\in\mathcal{Y}} \ln \frac{p(y)}{q(y)}
= 
\sup_{p\in \mathcal{M}} R(p,q),
\end{equation}
which stands for the ``maximal regret'' while using $q$ to code data related to any model \re{$p$} in $\mathcal{M}$~\cite{grunwald2007minimum}. 

A remarkable result from the MDL literature is that the minimiser of $\text{REG}_\text{max}$ can often be written in closed form, and is given by an NML distribution of the form
\begin{equation}
q_{\text{NML},\mathcal{M}}(y) = \frac{ \sup_{p\in\mathcal{M}} p(y) }{ Z_\mathcal{M} },
\end{equation}
where $Z_\mathcal{M} = \sum_{y \in \mathcal{Y}} \sup_{p\in\mathcal{M}} p(y)$ is a normalisation constant. The minimal regret is given by 
\begin{equation}
\min_{q} \text{REG}_\text{max}( \mathcal{M}, q ) 
= \text{REG}_\text{max}(\mathcal{M}, q_{\text{NML},\mathcal{M}})
= \ln  Z_\mathcal{M},
\end{equation}
being known as the \emph{stochastic complexity} of $\mathcal{M}$~\cite{Rissanen1996fisher}. %

Note that the NML might not be well-defined if $Z_\mathcal{M}$ diverges. One solution to those cases is to employ sub-models to reduce the minimal regret, since $\mathcal{M}' \subset \mathcal{M}$ implies $Z_{\mathcal{M}'} \leq Z_{\mathcal{M}}$. This approach is known as \textit{meta-universal} coding, which includes a range of techniques developed in the literature~\cite[Section 11.2]{grunwald2007minimum}.

\subsection{Quantifying information leakage}
\label{sec:leakage}

Consider a variable $\phi$ that parameterises the distributions $p_\phi(Y)$ that belong to  %
$\mathcal{M}$. 
We are interested in quantifying %
how much information about $\phi$ can be extracted 
from observations of $Y$. 
Note that this highly non-trivial issue %
is not properly addressed by naive applications of Shannon's mutual information or differential privacy criteria%
~\cite{du2012privacy,leakageLong}. 

We follow 
Ref.~\cite{leakage} and consider %
a random variable $U$ that is conditionally independent of $Y$ given $\phi$, and imagine guessing $U$ from $Y$ via $\hat{U}$, so that $U-\phi-Y-\hat{U}$ forms a Markov chain. Then, the \emph{maximal leakage} between $\phi$ and $Y$,
\begin{equation}\label{eq:leakage0}
\mathcal{L}(\phi \rightarrow Y)    \coloneqq
\sup_{U-\phi-Y-\hat{U}} \log 
\frac{ \mathbb{P}\{U=\hat{U}\} }{\max_{u\in\mathcal{U}} \mathbb{P}\{U=u\}},
\end{equation} 
characterizes the least protected secret $U$ (that is, the worst case over $U$) of $\phi$ with respect to $Y$. A closed-form formula for $\mathcal{L}(\phi \rightarrow Y)$ is given by~\cite[Corollary 4]{leakageLong}
\begin{align}\mathcal{L}(\phi\to Y) 
= \log \sum_{y\in\mathcal{Y}} \:\sup_{\theta\in\text{supp}(\phi)} f(y|\theta), 
\label{leakageContinous} 
\end{align}
with $\text{supp}(\phi) \coloneqq \{\theta\in\Theta: \mathbb{P}\{\phi=\theta\}>0\}$. This form is equivalent to the Sibson's mutual information of order infinity~\cite{infoRadius}, and has a number of useful properties and an operational interpretation that are discussed in Ref.~\cite{leakageLong}.

\section{Optimizing the hypothesis based on meta-universal coding principles}
\label{sec:main_results}

\subsection{Learning based on universal source coding}
\label{sec:learning_nml}

We first focus on a parametric model $\mathcal{P}$, which is a set of conditional distributions $p_{\bm\theta}(y|x)$ indexed by  $\bm\theta=(\theta_1,\dots,\theta_d) \in \mathbb{R}^d$. Following meta-universal coding principles (c.f. Section~\ref{sec:universal_compression}), we consider sub-models of the form 
\begin{equation}
\mathcal{A}(z^n) = \left\{p_{\bm\theta}(\cdot|\cdot) \in \mathcal{P}: \bm\theta \in \Theta(z^n)\right\} \subset \mathcal{M},    
\end{equation}
where $\Theta(z^n) \subset \mathbb{R}^d$ is a restriction in the space of parameters that depends on the traning set $z^n$. For the sub-model $\mathcal{A}(z^n)$, we define the following NML distribution:
\begin{equation}\label{eq:meta_NML}
q_{\text{NML},\mathcal{A}}(y|x,z^n) \coloneqq \frac{ \sup_{\bm \theta\in\Theta(z^n)} p_{\bm\theta}(y|x) }{ Z\big( x; \Theta(z^n) \big) },
\end{equation}
with $Z \big(x; \Theta(z^n)\big) = \sum_{y \in \mathcal{Y}} \sup_{\bm\theta \in \Theta(z^n)} p_{\bm\theta}(y|x)$. \re{Please note that this type of NML construction has been considered before in Ref.~\cite[Sec. 5]{Fogel2019individualdata}.}
Importantly, $Z \big(x; \Theta(z^n)\big) < \infty$ due to the finiteness of $\mathcal{Y}$, and hence $q_{\text{NML},\mathcal{A}}$ is well-defined for all $\mathcal{A}(z^n)$. The minimal regret attained by this NML distribution is $\ln Z \big(x; \Theta(z^n)\big)$, which corresponds to the stochastic complexity of model $\mathcal{A}(z^n)$.

When designing an NML distribution, choosing an adequate sub-model $\mathcal{A}(z^n)$ is critical --- or, equivalently, to set adequate parameter restrictions $\Theta(z^n)$. To gain insight about the effect of $\Theta(z^n)$ on the corresponding NML distribution, let us study the stochastic complexity of the sub-model as a form of information leakage (c.f. Section~\ref{sec:leakage}). For this, we consider a random variable $\bm\phi$ that takes values in a subset of the parameter space $\Theta(z^n)\subset \mathbb{R}^d$, 
and assume it satisfies 
\re{the Markov chain $\bm\phi - Z^n - X$}. Following Eq.~\eqref{leakageContinous}, the maximal leakage from $\bm\phi$ to $Y$ for given $X=x$ and $Z^n=z^n$ is%
\begin{equation}
\mathcal{L}(\bm\phi \rightarrow Y|x;z^n) : = \ln \left\{ \sum_{y\in\mathcal{Y}}\: \sup_{\bm\theta\in\text{supp}( \bm\phi|z^n)} p_{\bm\theta}(y|x) \right\},
\end{equation}
with $\text{supp}(\bm\phi|z^n) = \{ \bm\theta \in \Theta(z^n): \prob{ \bm\phi = \bm\theta|Z^n = z^n} > 0 \}$. 
This quantity has two useful properties:
\begin{itemize}
    \item[1.] \textit{It corresponds to a stochastic complexity:} if $\bm\phi$ is such that $\text{supp}(\bm\phi|z^n)=\Theta(z^n)$, then $\mathcal{L}(\bm\phi \rightarrow Y|x;z^n)=\log Z\big(x;\Theta(z^n)\big)$.
    \item[2.] \textit{It is monotonous with $\text{supp}(\bm\phi|z^n)$, and does not depend on other details of its distribution}: if $\bm\phi_1$ and $\bm\phi_2$ are variables such that $\text{supp}(\bm\phi_1|z^n) \subseteq \text{supp}(\bm\phi_2|z^n)$, then
    $\mathcal{L}(\bm\phi_1 \rightarrow Y|x;z^n) \leq \mathcal{L}(\bm\phi_2 \rightarrow Y|x;z^n)$.
\end{itemize}

Intuitively, $\mathcal{L}(\bm\phi \rightarrow Y|x;z^n)$ quantifies the information about $\bm\phi$ that can still be leaked from $Y$ after $x$ and $z^n$ have already been  given.\footnote{Note that $\mathcal{L}(\bm\phi \rightarrow Y|x;z^n)$ is not a conditional leakage, but the leakage for given values of $X=x$ and $Z^n=z^n$. Conditional leakage has been defined in Ref.~\cite{leakageLong}.} 
Put simply, the leakage measures how much better the training would be with $n+1$ samples, by considering all potential additional training samples of the form $z_{n+1}=(x,c_k)$ with $k=1,\dots,K$. Therefore, a high value of $\mathcal{L}(\bm\phi \rightarrow Y|x;z^n)$ implies that the training enabled by $z^n$ has not saturated yet and still has room for improvement.

We make this intuition precise with the analysis carried out below.
Let us denote by $q_{\text{NML},\bm\phi}(y|x,z^n)$ the NML distribution for the model $\mathcal{P}$ with parameters restricted to $\text{supp}(\bm\phi|z^n)$, and consider the hypothesis  %
given by
\begin{align}
h_{\text{NML},\bm\phi}(x,z^n) 
&= \argmax_{y\in\mathcal{Y}} q_{\text{NML},\bm\phi}(y|x,z^n) \\
&= \argmax_{y\in\mathcal{Y}} \sup_{\bm\theta \in \text{supp}(\bm\phi|z^n)} p_{\bm\theta}(y|x). 
\end{align}
Our first result identifies upper bounds to the performance of this hypothesis.
\begin{theorem}\label{prop:bound_nice}
Consider a $d$-dimensional parametric model $\mathcal{P}$, and a conditional probability $f(y|x)$. Then, for any random variable $\bm\phi\in\mathbb{R}^d$ that depends on a dataset $z^n\in\mathcal{Z}^n$, the following bound holds:
\begin{align}
  \textnormal{\texttt{E}}&(h_{\textnormal{NML},\bm\phi}; x,z^n)  - \textnormal{\texttt{E}}( h_{\textnormal{MAP}}; x) \nonumber \\
  &\leq  \exp\big\{ \Delta\big(f, \textnormal{supp}(\bm\phi|z^n)\big|x\big) + \mathcal{L}(\bm\phi \rightarrow Y|x;z^n) \big\} - 1~,\nonumber
\end{align}
where $\Delta(f, \Theta |x) \coloneqq
\inf_{\bm \theta\in \Theta} \max_{y\in\mathcal{Y}} \ln \frac{f(y|x)}{p_{\bm\theta}(y|x)}~$. %
\end{theorem}
\begin{proof}
The proof proceeds in three steps. First, one proves that for any distribution $q(y|x,z^n)$ the following bound holds:
\begin{equation}\label{eq:upperbound}
\texttt{E}(h_q; x,z^n)  - \texttt{E}(h_{\textnormal{MAP}};{x}) \leq  e^{R(f,q|x,z^n)} - 1~,
\end{equation}
where $R(f,q|x,z^n)\coloneqq \max_{y\in\mathcal{Y}} \ln \frac{ f(y|x) }{ q(y|x,z^n)}$ is the redundancy between $f$ and $q$ given $x$ and the training sample $z^n$ 
(c.f. Section~\ref{sec:universal_compression}). 
Then, one proves a triangle inequality $R(f,q) \leq \Delta(f,\Theta) + \text{REG}_\text{max}(\mathcal{A},q)$ for any sub-model $\mathcal{A}$ with parameters in $\Theta\subseteq \mathbb{R}^d$. Finally, the two previous steps are combined using $q = q_{\text{NML},\bm\phi}$ and $\Theta = \text{supp}(\bm\phi|z^n)$ to show the desired result. The details of the proof can be found in Appendix~\ref{app:bound_nice}.
\end{proof}

The above result reflects the trade-offs involved in the design of $h_{\text{NML},\bm\phi}$: on the one hand, %
having a variable $\bm\phi | z^n$ with a large support provides a big model which reduces $\Delta$,
at the risk of introducing a substantial regret as measured by the leakage $\mathcal{L}$; %
on the other hand, having a reduced support of $\bm\phi | z^n$ guarantees a small leakage, at the price of increasing $\Delta$. %
This result shows, in turn, that the maximal leakage provides a natural measure of overfitting. In effect, if the model with variables in $\text{supp}(\bm\phi|z^n)$ is too large, then for each class $c_k$ there exists a parameter $\bm\theta_k\in\text{supp}(\bm\phi|z^n)$ such that $p_{\bm\theta}(c_k|x) \approx 1$, and hence $\mathcal{L}\approx \log |\mathcal{Y}|$. This is an indication of overfitting, as --- rewording Ref.~\cite[Ch. 6]{shalev2014understanding} --- a hypothesis that can accommodate every possible outcome explains none of them. On the other extreme, if $\argmax_{\bm\theta\in\text{supp}(\bm\phi|z^n)} p_{\bm\theta}(y_k|x)$ is approximately constant for all classes, then $\mathcal{L}\approx 0$, which implies that the hypothesis is trustable.

We conclude this section by presenting a method to bound $\mathcal{L}(\bm\phi \rightarrow Y|x;z^n)$ when the Fisher information matrix of the family $\mathcal{P}$ is well-defined. The Fisher information matrix of the distribution $p_{\bm\theta}(y|x)$ can be defined to be the $d\times d$ matrix $I(\bm\theta | x)$ whose component in the $i$-th row and $j$-th column is calculated as
\begin{equation}
\big[ \: I (\bm\theta| x) \: \big]_{i,j} = \mathbb{E}\left\{ \frac{\partial}{ \partial \theta_i} \ln p_{\bm\theta}(Y|x) \cdot \frac{\partial}{ \partial \theta_j} \ln p_{\bm\theta}(Y|x) \right\}~.
\end{equation}
The maximal eigenvalue of $I (\bm\theta| x)$ is denoted as $\sigma_\text{max}(\bm\theta|x)$. 
\begin{lemma}\label{eq:bound_fisher}
If $\textnormal{supp}(\bm\phi|z^n)$ is a convex set and the Fisher information matrix is well-defined, then
\begin{equation}
\mathcal{L}(\bm\phi\rightarrow Y|x;z^n) 
\leq 
\ln \left\{ 1 + \sum_{k=2}^K || \bm\theta_k- \bm\theta_1|| \sqrt{  \sigma_\textnormal{max}(\tilde{\bm\theta}_k|x) }  \right\},\nonumber
\end{equation}
with $\bm\theta_i = \argmax_{\bm\theta \in \textnormal{supp}(\phi|z^n)} p_{\bm\theta} (y_i|x)$ for $i=1,\dots,K$ with $\mathcal{Y} = \{y_1,\dots,y_K\}$, and $\tilde{\bm\theta}_j = \tau_j \bm\theta_1 + (1-\tau_j) \bm\theta_j$ with $\tau_j \in [0,1]$ for $j=2,\dots,K$.
\end{lemma}
\begin{proof}
See Appendix~\ref{App-proof-eq:bound_fisher}.
\end{proof}

\subsection{Learning guarantees for well-specified models}
\label{sec:pac_identifiable}

We now consider the case where there exists a set of parameters $\bm\theta_0\in\Theta\subset\mathbb{R}^d$ such that $f(y|x) = p_{\bm\theta_0}(y|x)$. Let us focus on the case where there is a consistent estimator $\hat{\theta}: \mathcal{Z}^n\to\Theta$ such that $\hat{\theta}(Z^n) \xrightarrow{p} \theta_0$. Our next result is that, under these conditions, there exists a sequence of random variables $\bm\phi_n$ such that the hypothesis $h_{\text{NML},\bm\phi_n}$ attains a form of agnostic probably approximately correct (PAC) learning~\cite{haussler1990probably,shalev2014understanding}.

\begin{theorem}\label{teo:pac}
Consider $f(y|x) = p_{\bm\theta_0}(y|x) \in \mathcal{P}$ for some unknown parameter $\bm\theta_0 \in \Theta\subset \mathbb{R}^d$, and assume that there exists a consistent estimator $\hat{\theta}(Z^n)$ of $\bm\theta_0$. Also, assume that the Fisher matrix of $\mathcal{P}$ is well-defined over all $\Theta$, and that $\bm\theta_0$ is an interior point. Then, for given $x\in\mathcal{X}$ and $\epsilon, \delta >0$, there exists a random mapping $\bm\phi| \hat{\theta} $ and $n_0\in\mathbb{N}$ such that
\begin{equation}
  \textnormal{\texttt{E}}(h_{\textnormal{NML},\bm \phi};x,z^n) \leq \textnormal{\texttt{E}}(h_\textnormal{MAP};x) + \epsilon
\end{equation}
for all $n\geq n_0$, where the inequality holds for all $z^n\in B\subset \mathcal{Z}^n$ with $\prob{Z^n\in B} \geq 1-\delta$.
\end{theorem}

\begin{proof}
One builds $\bm\phi$ as a noisy version of a consistent estimator $\hat{\theta}(z^n)$, with the noise regulated by a parameter $\rho$. By carefully choosing $\rho$, one can use Theorem~\ref{prop:bound_nice} and bound $\Delta$ using the properties of the consistent estimator, and control the leakage $\mathcal{L}$ using Lemma~\ref{eq:bound_fisher}. The full proof is presented in  Appendix~\ref{app:pac}.
\end{proof}

\begin{corollary}[Heuristic PAC learning]\label{cor:pac}
If the assumptions required by Theorem~\ref{teo:pac} hold, 
then for given $\delta,\epsilon >0$ there exists a random mapping $\bm\phi | \hat{\theta}$ and an $n_0$ such that
\begin{equation}
\E{ \textnormal{\texttt{E}}(h_{\textnormal{NML},\bm\phi};X,z^n) }  \leq  \E{ \textnormal{\texttt{E}}(h_{\textnormal{MAP}};X) } + \epsilon
\end{equation}
for all $n\geq n_0$, where the inequality holds for all $z^n\in B\subset \mathcal{Z}^n$ with $\prob{Z^n\in B} \geq 1-\delta$.
\end{corollary}
\begin{proof}
See Appendix~\ref{app:cor_pac}.
\end{proof}

The conditions of Theorem~\ref{teo:pac} are satisfied if $\mathcal{P}$ is an exponential family (i.e. $p_{\bm\theta}(y|x)$ is an exponential family distribution for each $x\in\mathcal{X}$). Also, if $|\mathcal{X}| < \infty$ then any conditional distribution $f(y|x)$ is just a collection of $2^{|\mathcal{X}|}$ multinomial distributions, and hence can be expressed using $|\mathcal{Y}| \cdot 2^{|\mathcal{X}|}$ parameters. In both cases, the corresponding parameters can be estimated via a maximum likelihood estimator, which is known to be consistent in these cases.\footnote{For more information about existence of consistent estimators, see~\cite{King1999Mathematical}.} Please note that it is not straightforward to use our proof techniques to guarantee heuristic PAC learning to classification based directly on $\hat{\bm\theta}$ (see Appendix~\ref{app:plug-in}).

It would be useful to find explicit expressions for the dependency of $\delta,\epsilon$ and $n_0$. %
For the particular case of models with a maximum likelihood estimator (MLE), one can prove additional properties of the $h_{\text{NML},\bm\phi}$ hypothesis. We leverage the fact that MLEs follow a central limit theorem: %
\begin{equation}\label{eq:CLT}
\sqrt{n} \Big( \hat{\theta}(z^n) - \bm\theta_0 \Big) \xrightarrow{d} N\big(0,I^{-1}(\bm\theta_0) \big),
\end{equation}
with $I(\bm\theta) = \mathbb{E}\{ I(\bm\theta|X) \}$ being the unconditional Fisher matrix (with the average taken over  both $Y$ and $X$).

\begin{proposition}\label{prop:MLE}
Consider a $d$-dimensional parametric model $\mathcal{P}$ with well-defined MLE $\hat{\theta}(z^n)$ and a positive-definite Fisher matrix $I(\bm\theta)$. 
Then, for given $\delta >0$, $x\in\mathcal{X}$ and $z^n\in\mathcal{Z}^n$, the following holds:
\begin{align}
\textnormal{\texttt{E}}(h_{\textnormal{NML},\bm \psi};x,z^n)  - \textnormal{\texttt{E}}(h_\textnormal{MAP};x) 
&\leq  e^{\mathcal{L}(\bm\psi\rightarrow Y|x;z^n)} - 1  \nonumber\\
&\leq \frac{1}{\sqrt{n}} K_{\delta,x}~, 
\label{eq:ineqs}
\end{align}
where $\bm\psi \coloneqq \hat{\theta}(z^n) + W_\rho \in \mathbb{R}^d$ with $W_\rho$ uniformly distributed over a ball of radius 
$\rho = \mathcal{O}(n^{-1/2})$ 
and $K_{\delta,x}$ is a constant that does not depend on $n$.
\end{proposition}
\begin{proof}
See Appendix~\ref{app:proof_MLE}.
\end{proof}

Above, the first inequality provides a practical way to estimate the performance gap between $h_{\text{NML},\bm\psi}$ and $h_\text{MAP}$. In effect, given that the radius $\rho$ of the noise term of $\bm\psi$ has an explicit value, one can estimate the leakage $\mathcal{L}$.
Additionally, the second inequality states that the performance gap reduces at least as $1/\sqrt{n}$ with the number of training samples.

\subsection{Learning non-identifiable systems}

In the previous section, we studied the PAC learning properties of NML estimators in the somewhat restrictive scenario in which the target function $f(y|x)$ belongs to the parametric family of models under consideration. This final subsection provides a generalisation of the main results presented above to more widely applicable settings.

We now consider a family of parametric models $\mathcal{P}$ that is capable of universal approximation, in the sense of Hornik~\cite{hornik1991approximation}:
in particular, for a given $f(y|x)$ with reasonable properties and $\epsilon>0$, we assume that there exists a subset of parameter space $\Theta_{f,\epsilon}\subset\mathbb{R}^d$ such that $R(f,p_{\bm\theta}) < \epsilon$ for all $\bm\theta\in\Theta_{f,\epsilon}$.\footnote{To see why a universal approximator satisfies $R(f,p_{\bm\theta}) < \epsilon$, consider Theorem 1 in Ref.~\cite{cybenko1989approximation}, stating that for any given target function $g(x)$, a parametrised approximator $G_{\bm\theta}(x)$, and an $\epsilon > 0$ there exists $\bm\theta$ such that $|g(x) - G_{\bm\theta}(x)| \leq \epsilon$ for all $x$. Then, consider $g = \ln f$ and $G_{\bm\theta} = \ln p_{\bm\theta}$ to obtain the desired bound on $R(f, p_{\bm\theta})$.} 
Additionally, we consider that the system may be non-identifiable~\cite{white1989learning}, in the sense that there are multiple $\bm\theta$ that minimise $\texttt{E}(h_{p_{\bm\theta}})$, and in general the set $\Theta_{f,\epsilon}\subset\mathbb{R}^d$ might be non-convex. 
Moreover, we assume that there exists a (non-ergodic) estimator that converges to $\Theta_{f,\epsilon}$ in probability for any $\epsilon>0$; i.e. a function $\hat{\theta}:\mathcal{Z}^n\to \mathbb{R}^d$ such that for all $\delta,\rho>0$ there exists an $n_0(\delta,\rho)\in\mathbb{N}$ such that for all $n>n_0$ there is a set $B\subset\mathbb{R}^d$ of measure $\mathbb{P}\{ Z^n \in B\} > 1-\delta$ such that $\{\bm\theta\in\mathbb{R}^d: || \bm\theta - \hat{\theta}(z^n)  || < \rho \} \cap \Theta_{f,\epsilon} \neq \emptyset$ for all $z^n\in B$. The next result shows that the desirable properties of our NML strategy still hold in this more general context.

\begin{proposition}\label{prop:non}
Consider a conditional probability $f(y|x)$, and a universal approximator model $\mathcal{P}$ with well-defined Fisher matrix and a non-ergodic estimator $\hat{\theta}$ that converges in probability to $\Theta_{f,\epsilon}$ for any $\epsilon>0$. Then, given $x\in\mathcal{X}$ and $z^n\in\mathcal{Z}^n$, for each $\epsilon,\delta>0$, there exists $n_0\in\mathbb{N}$ and a random mapping $\bm\phi|\hat{\theta}$ such that for all $n>n_0$
\begin{equation}
  \textnormal{\texttt{E}}(h_{\textnormal{NML},\bm \phi};x,z^n)  \leq \textnormal{\texttt{E}}(h_\textnormal{MAP};x)  + \epsilon~.
\end{equation}
\end{proposition}

\begin{proof}
See Appendix~\ref{app:non}.
\end{proof}

This result generalises the main result in Theorem~\ref{teo:pac} to the more practical setting of large non-identifiable models, like multi-layer neural networks, showing that NML can provide PAC guarantees even in the case of very general models.

\section{Conclusion}
\label{sec:conc}

This paper provides a first step in the exploration of the potential of meta-universal coding and maximal leakage techniques for supervised learning theory. We have proposed an approach to build hypotheses based on Normalised Maximum Likelihood (NML) that can be applied to any standard learning algorithm. Crucially, we showed that models evaluated with this NML strategy attain heuristic PAC learning in a wide variety of contexts, and for specific cases we further showed that the performance gap between the NML approach and the optimal strategy decreases at least with the square-root of the number of samples.

In addition, we have provided an upper bound on the performance of our proposed NML strategy, and showed that this upper bound is directly determined by maximal leakage: a quantity used in the data privacy literature that we linked to the model's capacity to overfit. %
One interesting aspect of maximal leakage as a measure of overfitting is that it depends on the specific input to be classified, and hence could potentially be used to assess open problems in adversarial learning settings. 

We hope this contribution may motivate further research efforts within the fascinating interface between learning, universal compression, and data privacy.

\section*{Acknowledgment}

The authors thank Amedeo Esposito and Ibrahim Issa for inspiring discussions, and Yike Guo for supporting this research.

\appendices

\section{Proof of Theorem~\ref{prop:bound_nice}}
\label{app:bound_nice}

\begin{proof}
Consider the model class $\mathcal{M}=\{p_{\bm\theta}\in\mathcal{P}: \bm\theta \in \text{supp}(\bm\phi|z^n)\}$. By using Lemmas~\ref{lemma:ineq} and \ref{lemma:RR} (shown below) one can show that
\begin{align}
\texttt{E}(h_{\text{NML},\bm\phi}; x,z^n)  &- \texttt{E}( h_{\text{MAP}}; x)
\leq  
\exp\big\{ 
\Delta\big(f, \text{supp}(\bm\phi|z^n)\big) \nonumber \\
&+ 
\text{REG}_\text{max}\big( \text{supp}(\bm\phi|z^n), q_{\text{NML},\bm\phi}|x,z^n\big) \big\} - 1.\nonumber
\end{align}
The Theorem is then proven by noting that
\begin{align}
\text{REG}_\text{max}\big( \text{supp}(\bm\phi|z^n), q_{\text{NML},\bm\phi} | x,z^n \big)
&=
\ln \left\{Z\big(x; \text{supp}(\bm\phi|z^n)\big) \right\} \nonumber\\
&=
\mathcal{L}(\bm\phi \rightarrow Y|x;z^n),\nonumber
\end{align}
with $Z\big(x; \text{supp}(\bm\phi|z^n)\big)$ as defined in Eq. \eqref{eq:meta_NML}.
\end{proof}

\begin{lemma}\label{lemma:ineq}
For $h_\textnormal{MAP}$ and $h_q$ as defined in Eqs.~\eqref{eq:MAP} and~\eqref{eq:q}, the following holds:
\begin{equation}
\textnormal{\texttt{E}}(h_q; x,z^n)  - \textnormal{\texttt{E}}(h_{\textnormal{MAP}};{x}) \leq  e^{R(f,q|x,z^n)} - 1,\nonumber
\end{equation}
with $R(f,q|x,z^n)\coloneqq \max_{y\in\mathcal{Y}} \ln \frac{ f(y|x) }{ q(y|x,z^n)}$~.
\end{lemma}
\begin{proof}
Let us use $\delta \coloneqq R(f,q|x,z^n)$ as a shorthand notation throughout the proof. Then, $\ln f(y|x) \leq \delta + \ln q(y|x)$ for all $y\in\mathcal{Y}$. Then, one can show that
\begin{align}
\prob{ Y = h_\text{MAP}(X) | X=x} &= f( h_\text{MAP}(x) |x) \nonumber\\
& \leq e^\delta q( h_\text{MAP}(x) |x,z^n) \nonumber\\
& \leq e^\delta q( h_q(x,z^n) |x,z^n) ,\nonumber
\end{align}
where the last equality holds because $h_q(x,z^n) = \argmax_{y \in \mathcal{Y}} q(y|x,z^n)$. Now, note that for all $y_0\in\mathcal{Y}$ one has that
\begin{align}
q(y_0|x,z^n) &= 1 - \sum_{y\neq y_0} q(y|x,z^n) \nonumber \\
& \leq 1 - e^{-\delta} \sum_{y\neq y_0} f(y|x) \nonumber \\
& = 1 - e^{-\delta} \Big( 1 - f(y_0|x) \Big) .\nonumber
\end{align}
Then, this gives
\begin{align}
\mathbb{P}\big\{ Y = &h_\text{MAP}(X) | X=x \big\}  \leq e^\delta \Big[ 1 - e^{-\delta} + e^{-\delta} f\big( h_q(x,z^n) |x\big) \Big] \nonumber\\
&= e^\delta - 1 + \prob{ Y = h_q(X,Z^n) | X=x,Z^n=z^n},\nonumber
\end{align}
from where the desired result follows.
\end{proof}

Note that $R(f,q|x,z^n) \geq 0$ and hence $e^R-1$ is non-negative, which is consistent with the optimality of the MAP hypothesis. 

\begin{lemma}\label{lemma:RR}
For any model class $\mathcal{M}$, the following bound holds:
\begin{equation}
R(q,f|x,z^n) \leq \Delta(f, \mathcal{M}|x) + \textnormal{REG}_\textnormal{max}(q, \mathcal{M} | x,z^n),\nonumber
\end{equation}
with $R(q,f|x,z^n)$ as defined in Lemma~\ref{lemma:ineq}.
\end{lemma}
\begin{proof}
First, note that
\begin{equation}
R(q,f|x,z^n) = \max_{y\in\mathcal{Y}} \Big\{ \ln \frac{f(y|x)}{p(y|x)} + \ln \frac{p(y|x)}{q(y|x,z^n)} \Big\},\nonumber
\end{equation}
which holds for all $p \in \mathcal{M}$. This implies that
\begin{align}
R(q,f|x,z^n) &= \inf_{p \in \mathcal{M}} \max_{y\in\mathcal{Y}} \Big\{ \ln \frac{f(y|x)}{p(y|x)} + \ln \frac{p(y|x)}{q(y|x,z^n)} \Big\} \nonumber\\
&\leq \inf_{p \in \mathcal{M}} \max_{y\in\mathcal{Y}} \ln \frac{f(y|x)}{p(y|x)} + \sup_{p \in \mathcal{M}} \max_{y\in\mathcal{Y}} \ln \frac{p(y|x)}{q(y|x,z^n)},\nonumber
\end{align}
proving the desired result. Note that, above, the last inequality is a consequence of the fact that 
\begin{align}
\inf_x \{ f(x) + g(x) \} &\leq \inf_x \{ f(x) + \sup_x g(x) \} \nonumber\\
&= \inf_x f(x) + \sup_x g(x).\nonumber
\end{align}
\end{proof}

\section{Proof of Lemma~\ref{eq:bound_fisher}}\label{App-proof-eq:bound_fisher}

\begin{proof}
For the second part, %
let us enumerate the possible classes as $\mathcal{Y} = \{y_1,\dots,y_K\}$. Now, for given training data $z^n\in \mathcal{Z}^n$, we introduce the shorthand notation $\bm\theta_k \coloneqq \argmax_{\bm\theta \in \text{supp}(\bm\phi|z^n)} p_{\bm\theta} (y_k|x)$ for $k=1,\dots,K$. Then,
\begin{align}
\exp\big\{ \mathcal{L}(\bm\phi\rightarrow Y|x,z^n) \big\} &= \sum_{k=1}^K p_{\bm\theta_k}(y_k|x) \nonumber\\
&= 1 + \sum_{k=2}^K \Big[ p_{\bm\theta_k}(y_k|x) - p_{\bm\theta_1}(y_k|x) \Big] \nonumber\\
&\leq 1 + \sum_{k=2}^K 2 d_{\text{TV}} \big( p_{\bm\theta_k}(y|x), p_{\bm\theta_1}(y|x) \big) \nonumber\\
&\leq 1 + \sum_{k=2}^K \sqrt{ 2  D \big( p_{\bm\theta_k}(y|x) || p_{\bm\theta_1}(y|x) \big) }.\nonumber
\end{align}
Above, $d_{\text{TV}}\big(p(y|x),q(y|x)\big) \coloneqq 1/2 \sum_{y\in\mathcal{Y}} \big| p(y|x) - q(y|x) \big|$ is the total variation distance, and the last inequality is a direct application of the well-known Pinsker inequality. To finish the proof, note that
\begin{align}
\partial_i D\big( p_{\bm\theta}(Y|x) || p_{\bm\theta_0}(Y|x) \big) \Big|_{\bm\theta =\bm\theta_0} &= 0,\nonumber\\
\partial^2_{i,j} D\big( p_{\bm\theta}(Y|x) || p_{\bm\theta_0}(Y|x)) \Big|_{\bm\theta =\bm\theta_0} &= I_{i,j} (\bm\theta_0| x) .\nonumber
\end{align}
Therefore, a first order Taylor expansion of the Kullback-Leibler divergence on $\bm\theta$ centered in $\bm\theta_0$ that expresses the reminder according to the Lagrange form~\cite{lang2012calculus} gives
\begin{equation}
D \big( p_{\bm\theta_k}(Y|x) || p_{\bm\theta_1}(Y|x)) = \frac{1}{2}(\bm\theta_k - \bm\theta_1)^T I(\tilde{\bm\theta}| x) (\bm\theta - \bm\theta_0),\nonumber
\end{equation}
where $\tilde{\bm\theta}_k = \tau_k \bm\theta_1 + (1-\tau_k) \bm\theta_k$ for some $\tau_k \in (0,1)$. Note that $\tilde{\bm\theta}_k \in \text{supp}(\bm\phi|x,z^n)$ due to the convexity of the latter. The proof concludes by noting that
\begin{equation}
(\bm\theta - \bm\theta_0)^T I(\tilde{\bm\theta}; x) (\bm\theta - \bm\theta_0) \leq ||\bm\theta - \bm\theta_0||^2 \sigma_\text{max}(\tilde{\bm\theta}|x), \nonumber
\end{equation}
due to the properties of the maximal eigenvalue $\sigma_\text{max}(\tilde{\bm\theta}|x)$.
\end{proof}

\section{Complete proof of Theorem~\ref{teo:pac}}
\label{app:pac}

\begin{proof}

Let us consider a given $x\in\mathcal{X}$. As $\hat{\theta}$ is a consistent estimator of $\bm\theta_0$, then for given $\delta,\rho >0$ there exists $n_{\bm\theta}(\delta,\rho) \in \mathbb{N}$ such that for all $n\geq n_{\bm\theta}(\delta,\rho)$ the following holds:
\begin{equation}
\prob{ || \hat{ \theta}(Z^n) - \bm\theta_0 || \geq \rho } < \delta~.\nonumber
\end{equation}
This implies that $B \coloneqq \{ z^n \in \mathcal{Z}^n : || \hat{ \theta}(z^n) - \bm\theta_0 || < \rho \}$ satisfies $\prob{ Z^n \in B} \geq 1-\delta$. Also, by defining $\bm\phi = \hat{\theta}(Z^n) + W_\rho$ with $W_\rho$ distributing uniformly over $B(\rho)  = \{\bm\theta\in\mathbb{R}^d: ||\bm\theta || < \rho \}$, then $\bm\theta_0 \in \text{supp}(\bm\phi|z^n)$ for all $z^n\in B$. This implies, in turn, that $\Delta( f, \text{supp}(\bm\phi|z^n)|x) = 0$. Therefore, using Theorem~\ref{prop:bound_nice} one finds that for all $z^n\in B$ the following inequality holds:
\begin{equation}\label{eq:ec1}
\texttt{E}(h_{\text{NML}, \bm\phi};x,z^n)  - \texttt{E}(h_{\text{MAP}};x) \leq \exp\big\{ \mathcal{L}(\bm\phi\rightarrow Y|x;z^n) \big\} - 1.
\end{equation}

To build a bound on $\mathcal{L}(\bm\phi\rightarrow Y|x;z^n)$, let us define
\begin{equation}\label{eq:sigma_rho}
\sigma_\text{max}^{(\rho)}\big(\bm\theta_0 |x \big) \coloneqq \sup_{||\bm\theta-\bm\theta_0 || < \rho} \sigma_\text{max}( \bm\theta | x). 
\end{equation}
By using the fact that $|| \tilde{\bm\theta} - \bm\theta || < 2\rho$ for any $\tilde{\bm\theta},\bm\theta\in\text{supp}(\bm\phi|z^n)$, a direct application of Lemma~\ref{eq:bound_fisher} shows that
\begin{equation}\label{eq:ec3}
\exp\big\{ \mathcal{L}(\bm\phi\rightarrow Y|x;z^n) \big\} \leq 
1 + 2\rho K \sqrt{ \sigma^{(\rho)}_\text{max}\big(\hat{\theta}(z^n) | x \big) }.
\end{equation}

Finally, for given $\delta, \epsilon >0$ one calculates 
\begin{equation}
\rho_\epsilon(x,z^n) = \min\{\epsilon, \epsilon/\texttt{C}_\epsilon (x;z^n) \}
\nonumber
\end{equation} 
with $\texttt{C}_\epsilon(x;z^n)\coloneqq 2K \sqrt{ \sigma^{(\epsilon)}_\text{max}\big(\hat{\theta}(z^n) | x \big) }$, which is well defined  %
for small $\epsilon$sien la  as $\bm\theta_0$ is an interior point.
Then, noting that $\rho\leq \epsilon$ implies that $\sigma^{(\rho)}_\text{max}\big(\hat{\theta}(z^n|x) \big) \leq \sigma^{(\epsilon)}_\text{max}\big(\hat{\theta}(z^n|x) \big)$, one can find that for all $n\geq n_{\bm\theta}\big(\rho_\epsilon(x;z^n), \delta\big)$ it is guaranteed that
\begin{equation}
\texttt{E}(h_{\text{NML},\bm \phi};x,z^n)  - \texttt{E}(h_\text{MAP};x) \leq \epsilon,\nonumber
\end{equation}
where the inequality holds for all $z^n \in B$. 
\end{proof}

\section{Proof of Corollary~\ref{cor:pac} }
\label{app:cor_pac}

\begin{proof}
Let us denote as $\text{T}(\bm\theta|x) : = \sum_{i=1}^d \Big[ I(\bm\theta|x) \Big]_{i,i}$ the trace of $ I(\bm\theta|x) $, and $\text{T}(\bm\theta) = \mathbb{E}\{ \text{T}(\bm\theta|X) \}$. Moreover, let us define
\begin{equation}
\text{T}^{(\rho)}(\bm\theta_0|x) \coloneqq \sup_{||\bm\theta-\bm\theta_0 || < \rho} \text{T}(\bm\theta|x)~.
\end{equation}
Then, by considering Eqs.~\eqref{eq:ec1} and \eqref{eq:ec3} and noting that $\sigma_\text{max}(\bm\theta|x) \leq \text{T}(\bm\theta|x)$, one can show that
\begin{align}
\E{\texttt{E}(h_{\text{NML},\bm\phi};X,z^n)} - &\E{\texttt{E}(h_{\text{MAP}};X)} \nonumber\\
&\leq 2\rho K \E{  \sqrt{ \sigma^{(\rho)}_\text{max}\big(\hat{\theta}(z^n) ;X \big) } } \nonumber\\
&\leq 2\rho K \E{  \sqrt{ \text{T}^{(\rho)}\big(\hat{\theta}(z^n) ;X \big) } } \nonumber\\
&\leq 2\rho K \sqrt{ \text{T}^{(\rho)}\big(\hat{\theta}(z^n) \big) }~.\nonumber
\end{align}
The last step uses the well-known Jensen inequality. %
Finally, the corollary is proven by selecting $\rho_\epsilon(z^n) = \min\{\epsilon, \epsilon/\texttt{D}_\epsilon (z^n) \}$ with $\texttt{D}_\epsilon(z^n)\coloneqq 2K \sqrt{ \text{T}^{(\epsilon)}\big(\hat{\theta}(z^n) \big) }$.
\end{proof}

\section{Plug-in hypothesis does not guarantee\\ heuristic PAC learning}
\label{app:plug-in}

Consider the plug-in hypothesis, which corresponds to our NML strategy with $\bm\phi = \hat{\bm\theta}$ and hence $h_{\text{NML},\bm\phi} = h_{p_{\hat{\bm\theta}}}$. Here we show that our proof of heuristic PAC learning cannot be applied --- at least directly --- to this case.

Let us consider how Theorem~\ref{prop:bound_nice} could be used in this scenario. As in this case $\bm\phi$ defines a particularly narrow model, i.e. $\text{supp}(\bm\phi|z^n) = \{\hat{\bm\theta}\}$, then is direct to verify that $\mathcal{L}(\bm\phi\rightarrow Y|x,z^n) = 0 $ and $\Delta\big(p_{\bm\theta_0},\text{supp}(\bm\phi|z^n)|x\big) = \max_{y\in\mathcal{Y}} \ln f(y|x) / p_{\hat{\bm\theta}}(y|x)$. While the consistency of $\hat{\bm\theta}$ guarantees the convergence to zero of $\Delta\big(p_{\bm\theta_0},\text{supp}(\bm\phi|z^n)|x\big)$ for each $x\in\mathcal{X}$, guaranteeing stronger types of convergence (which would be needed to prove heuristic PAC learning) is not straightforward. In particular, notice that to guarantee the convergence of $\sup_{x\in\mathcal{X}}\Delta\big(f, \text{supp}(\phi|z^n)|x\big)$ to zero as $n$ grows, as one would need a function $C(\bm\theta)$ such that for large $n$ the following holds for all $x\in\mathcal{X},y\in\mathcal{Y}$:
\begin{equation}
\ln p_{\bm\theta_0}(y|x) - \ln p_{\hat{\bm\theta}}(y|x) \leq C(\bm\theta_0) \cdot || \bm\theta_0 - \hat{\bm\theta} ||.
\end{equation} 
However, if the cardinality of $\mathcal{X}$ is infinite, it is possible to build examples where no such $C(\bm\theta)$ exists, even if $p_{\hat{\bm\theta}}(y|x) \rightarrow p_{\bm\theta_0}(y|x)$ for each $x\in\mathcal{X},y\in\mathcal{Y}$. This is a consequence of the fact that the derivative of the logarithm is unbounded within the interval $(0,1)$.

\section{Proof of Proposition~\ref{prop:MLE}}
\label{app:proof_MLE}

\begin{proof}

Under appropriate assumptions, the MLE $\hat{\theta}(Z^n)$ satisfies the Berry-Esseen bound~\cite{Pfanzagl:1973}
\begin{align}
   \left| \prob{  \left\| \sqrt{n} I^{1/2}(\bm\theta_0) \left( \hat{\theta}(Z^n) - \bm\theta_0 \right) \right\|^2 \leq G_d^{-1}(t)  } - t \right| &\le \frac{c}{\sqrt{n}}, \nonumber
\end{align}
where $G_d(\cdot)$ is the CDF of the chi-squared distribution with $d$ degrees of freedom, and $c$ is an absolute constant.
Therefore, using the fact that
\begin{equation}
\sqrt{\sigma_\text{min}(\bm\theta_0)} \cdot || \hat{\theta}(Z^n) - \bm\theta_0 ||
\leq 
|| I^{1/2}(\bm\theta_0) \left( \hat{\theta}(Z^n) - \bm\theta_0 \right) || 
,\nonumber
\end{equation}
one can show that
\begin{align}
    \prob{  n\sigma_\text{min}(\bm\theta_0) \left\| \hat{\theta}(Z^n) - \bm\theta_0 \right\|^2 \leq G_d^{-1}(t)  } & \ge   t - \frac{c}{\sqrt{n}}
    .\nonumber 
\end{align}
Then, by taking $t = 1-\delta + c n^{-1/2}$, and assuming that $n$ is large enough so that $t\in [0,1]$, then one can find that
\begin{align}
    \prob{  \left\| \hat{\theta}(Z^n) - \bm\theta_0 \right\| 
    \leq 
    \sqrt{ \frac{ G^{-1}_d(1-\delta+c n^{-1/2}) } 
    { n \sigma_\text{min}(\bm \theta_0)} }  
    } & \ge   1 - \delta
    .\nonumber 
\end{align}
Note that $\sigma_\text{min}(\bm\theta_0)>0$ because $I(\bm\theta_0)$ is assumed to be positive definite.

Therefore, by considering $\bm\psi \coloneqq \hat{\theta}(z^n) + W$ with $W$ uniformly distributed over a ball of radius 
\begin{equation}
\rho_n \coloneqq \sqrt{ \frac{G^{-1}_d(1-\delta+c n^{-1/2})}{n\sigma_\text{min}(\bm\theta_0)}},
\end{equation}
then $\bm\theta_0\in\text{supp}(\bm\psi|z^n)$ for all $z^n\in B\subset\mathcal{Z}^n$ with $\prob{Z^n\in B} = 1-\delta$. Then, $\Delta(\text{supp}(\bm\psi|z^n),f)=0$ for all $z^n\in B$, and hence the first inequality in \eqref{eq:ineqs} can be proven using Theorem~\ref{prop:bound_nice}. 

For proving the second inequality, note that a direct application of Lemma~\ref{eq:bound_fisher} shows that
\begin{equation}
\exp\big\{ \mathcal{L}(\bm\psi\rightarrow Y|x;z^n) \big\} \leq 
1 +  
\frac{2}{\sqrt{n}} 
\cdot
\sqrt{ \frac{\sigma^{(\rho_n)}_\text{max}\big(\hat{\theta}(z^n) |x \big)} {\sigma_\text{min}(\bm \theta_0)} },
\nonumber
\end{equation}
with $\sigma^{(\rho)}_\text{max}$ defined as in Eq.~\eqref{eq:sigma_rho}. 
Furthermore, by noting that by construction of $\rho_n$ it is guaranteed that $|| \hat{\theta}(Z^n) - \bm \theta_0|| \leq \rho_n$ with probability $1-\delta$, then 
$\sigma^{(\rho_n)}_\text{max}\big(\hat{\theta}(z^n) |x \big) 
\leq
\sigma^{(2\rho_n)}_\text{max}\big(\bm \theta_0 |x \big)$, which in turn implies that 
\begin{equation}
\exp\big\{ \mathcal{L}(\bm\psi\rightarrow Y|x;z^n) \big\} \leq 
1 +  
\frac{2}{\sqrt{n}} 
\cdot
\sqrt{ \frac{\sigma^{(2\rho_n)}_\text{max}\big( \bm \theta_0 |x \big)} {\sigma_\text{min}(\bm \theta_0)} }.
\nonumber
\end{equation}
Finally, the proof concludes by noting that $\rho_n$, and hence also $\sigma^{(\rho_n)}_\text{max}$, decrease with $n$.

\end{proof}

\section{Proof of Proposition~\ref{prop:non}}
\label{app:non}

\begin{proof}
Let us consider $\epsilon,\delta>0$, and define $\bm\phi: = \hat{\theta}(Z^n) + W_\rho \in \mathbb{R}^d$ with $W_\rho$ distributed uniformly over a $d$-dimensional ball of radius $\rho>0$. By the properties of $\hat{\theta}$, there exists $n_0(\delta,\rho) \in\mathbb{N}$ such that for all $n>n_0$ then there exists $\bm\theta_0\in \Theta_f \cap \text{supp}(\bm\phi|z^n) $, for all $z^n\in B$ with $\mathbb{P}\{ Z^n \in B \} > 1 - \delta$. Then, it is direct to check that $\Delta\big(f, \text{supp}(\bm\phi|z^n) \big) < \epsilon_0$ for all $z^n \in B$. Additionally, following the proof of Theorem~\ref{teo:pac} (in particular, the derivation that leads to Eq.\eqref{eq:ec3}), one can check that the fact that $\text{supp}(\bm\phi|z^n)$ has a bounded support implies that 
\begin{equation}\label{eq:ec33}
\exp\big\{ \mathcal{L}(\bm\phi\rightarrow Y|x;z^n) \big\} \leq 
1 + 2\rho K \sqrt{ \sigma^{(\rho)}_\text{max}\big(\hat{\theta}(z^n) | x \big) },
\nonumber
\end{equation}
with $\sigma^{(\rho)}_\text{max}(\bm\theta|x)$ as defined in Eq.~\eqref{eq:sigma_rho}.

With these results at hand, let us now choose $\rho(x;z^n) = \min\{\epsilon_0, \epsilon_0/\texttt{C}_{\epsilon_0} (x;z^n) \}$ with $\texttt{C}_{\epsilon_0}(x;z^n)\coloneqq 2K \sqrt{ \sigma^{(\epsilon_0)}_\text{max}\big(\hat{\theta}(z^n) | x \big) }$. Using these results and Theorem~\ref{prop:bound_nice}, and the fact that $e^x\approx 1+x$ for $1\gg |x|$, one finds that for all $n > n_0\big(\delta,\rho(x;z^n)\big)$ then
\begin{align}
\texttt{E}(h_{\text{NML},\bm \phi};x,z^n)  - \texttt{E}(h_\text{MAP};x) 
\leq& e^{\Delta(f,\text{supp}(\bm\phi|z^n))} e^{ \mathcal{L}(\bm\phi\rightarrow Y|x;z^n) } \nonumber\\
&- 1 \nonumber\\
\leq& 2\epsilon_0 + \epsilon_0^2. \nonumber
\end{align}
Finally, the proof concludes by selecting $\epsilon_0$ such that 
\begin{equation}
    2\epsilon_0 + \epsilon_0^2 < \epsilon.
\end{equation}
\end{proof}

\bibliographystyle{IEEEtran}

\end{document}